\documentclass[]{amsart} 
\usepackage{amsmath,amsthm,amssymb,amsfonts,mathrsfs,color,hyperref, mathtools,crop, graphicx, enumitem, todonotes}
\usepackage[]{geometry}

\title[simplex symmetry in neural network classifiers]{On the emergence of simplex symmetry in the final and penultimate layers of neural network classifiers}

\theoremstyle{plain}
\begingroup
\newtheorem{theorem}{Theorem}[section]
\newtheorem*{theorem*}{Theorem}
\newtheorem*{"theorem"}{``Theorem''}
\newtheorem{corollary}[theorem]{Corollary}
\newtheorem{proposition}[theorem]{Proposition}

\newtheorem{lemma}[theorem]{Lemma}
\endgroup

\theoremstyle{definition}
\begingroup

\endgroup

\theoremstyle{remark}
\begingroup
\newtheorem{remark}[theorem]{Remark}

\endgroup 

\numberwithin{equation}{section}
\makeatletter
 \let\Ginclude@graphics\@org@Ginclude@graphics
\makeatother

\newcommand{\R}{\mathbb R} 

\renewcommand{\P}{{\mathbb P}}

\newcommand{\spt}{{\mathrm{spt}}}

\renewcommand{\H}{{\mathcal H}}

\newcommand{\F}{{\mathcal F}}

\newcommand{\X}{\mathcal{X}}

\newcommand{\Risk}{\mathcal{R}}

\newcommand{\LRa} {\Leftrightarrow}
\newcommand{\Ra} {\Rightarrow}

\renewcommand{\d}{\mathrm{d}}

\newcommand{\average}{{\mathchoice {\kern1ex\vcenter{\hrule height.4pt
width 6pt depth0pt} \kern-9.7pt} {\kern1ex\vcenter{\hrule
height.4pt width 4.3pt depth0pt} \kern-7pt} {} {} }}

\newcommand{\conv}{\mathrm{conv}}

\newcommand\showlabel{\addtocounter{equation}{1}\tag{\theequation}}

\DeclareMathOperator*{\argmin}{argmin} 
\DeclareMathOperator*{\argmax}{argmax}

\newcommand{\sw}[1]{\textcolor{black}{#1}}

\allowdisplaybreaks

 \makeatletter
\@namedef{subjclassname@2020}{2020 Mathematics Subject Classification}
\makeatother

\begin{document}

\author{Weinan E}
\address{Weinan E\\
Department of Mathematics and Program for Applied and Computational Mathematics\\
Princeton University\\
Princeton, NJ 08544
}
\email{weinan@math.princeton.edu}

\author{Stephan Wojtowytsch}
\address{Stephan Wojtowytsch\\
Program for Applied and Computational Mathematics\\
Princeton University\\
Princeton, NJ 08544
}
\email{stephanw@princeton.edu}

\date{\today}

\maketitle

\begin{abstract}%
A recent numerical study observed that neural network classifiers enjoy a large degree of symmetry in the penultimate layer. Namely, if $h(x) = Af(x) +b$ where $A$ is a linear map and $f$ is the output of the penultimate layer of the network (after activation), then all data points $x_{i, 1}, \dots, x_{i, N_i}$ in a class $C_i$ are mapped to a single point $y_i$ by $f$ and the points $y_i$ are located at the vertices of a regular $k-1$-dimensional \sw{standard simplex} in a high-dimensional Euclidean space. 

We explain this observation analytically in toy models for highly expressive deep neural networks. In complementary examples, we demonstrate rigorously that even the final output of the classifier $h$ is not uniform over data samples from a class $C_i$ if $h$ is a shallow network (or if the deeper layers do not bring the data samples into a convenient geometric configuration).
\end{abstract}



\section{Introduction}

A recent empirical study \cite{papyan2020prevalence} took a first step towards investigating the inner geometry of neural networks close to the output layer. In classification problems, the authors found that the data in the final and penultimate layers enjoy a high degree of symmetry. Namely, a neural network function $h_L:\R^d\to \R^k$ with $L$ layers can be understood as a composition
\begin{equation}
h_L(x) = A\,f_L(x) + b
\end{equation}
where $f_L:\R^d\to \R^m$ is (the composition of a componentwise nonlinearity with) a neural network with $L-1$ layers, $b\in \R^k$ and $A:\R^m\to \R^k$ is linear. In applications where $h_L$ was trained by stochastic gradient descent to minimize softmax-crossentropy loss to distinguish elements in various classes $C_1, \dots, C_k$, the authors observed that the following became approximately true in the long time limit.
\begin{itemize}
\item $f_L$ maps all elements in a class $C_i$ to a single point $y_i$.
\item The distance between the centers of mass of different classes in the penultimate layer $\|y_i- y_j\|$ does not depend on $i\neq j$.
\item Let $M= \frac1k \sum_{i=1}^k y_i$ be the center of mass of the data distribution in the penultimate center (normalizing the weight of data classes). Then the angle between $y_i - M$ and $y_j-M$ does not depend on $i\neq j$. 
\item The $i$-th row of $A$ is parallel to $y_i-M$.
\end{itemize}

In less precise terms, $h_L$ maps the classes $C_i$ to the vertices of a regular \sw{standard simplex} in a high-dimensional space. This phenomenon is referred to as `neural collapse' in \cite{papyan2020prevalence}. In this note, we consider the toy model where $f_L$ is merely a bounded measurable function and prove that under certain assumptions such simplex geometries are optimal. An investigation along the same lines has been launched separately in \cite{unconstrainedcollapse}.

Conversely, we show that even the output $h_L(C_i)$ of a shallow neural network $h_L$ over a data class $C_i$ does not approach a single value $z_i$ when the parameters of $h_L$ are trained by continuous time gradient descent. Since a deep neural network is the composition of a slightly less deep network and a shallow neural network containing the output layer, these results suggest that the $h_L$ cannot be expected to be uniform over a data class unless a convenient geometric configuration has already been reached two layers before the output.

We make the following observations.

\begin{enumerate}
\item Overparametrized networks can fit random labels at data points \cite{cooper2018loss} and can be efficiently optimized for this purpose in certain scaling regimes, see e.g.\ \cite{du2018bgradient,du2018gradient,weinan2019comparativepub}. The use of the class $\sw{L^\infty(\P; \R^m):= (L^\infty(\P))^m}$ as a proxy for very expressive deep neural networks thus can be justified heuristically from the static perspective of energy minimization (but not necessarily from the dynamical perspective of training algorithms). 

\sw{
In practice, the data distribution $\P$ is estimated on a finite set of sample points $\{x_1,\dots,x_N\}$ and an empirical distribution $\P_N = \frac1N \sum_{i=1}^N\delta_{x_i}$. A function $f_L\in L^\infty(\P; \R^m)$ determined by its values at the points $x_1,\dots, x_N$. A class of sufficiently complex neural networks which can fit any given set of outputs $\{y_1, \dots, y_N\}$ for inputs $\{x_1,\dots, x_N\}$ coincides with $L^p(\P; \R^m)$ for any $1\leq p\leq \infty$. The same is true for many other function models.
}

\sw{
If $\P_N = \frac1N \sum_{i=1}^N\delta_{x_i}$ or more generally, if all classes $C_1, \dots, C_k$ have a positive distance to each other, a function $f\in L^p(\P; \R^m)$ which is constant on every class can be extended to a $C^\infty$-function on $\R^d$. Thus in realistic settings, all functions below can be taken to be fairly regular.
}

\item As the softmax cross-entropy functional does not have minimizers in sufficiently expressive scaling-invariant function classes, we need to consider norm bounded classes. 

In the hypothesis class \sw{given by the ball of radius $R$ in $L^\infty(\P; \R^m)$}, the optimal map $h$ satisfies $h(x) = z_i$ for all $x$ in a data class $C_i$ and the values $z_i$ form the vertices of a regular \sw{simplex}. \sw{More precisely, the statement is valid under the constraint $\|h(x)\|_{\ell^p}\leq R$ for all $p\in (1,\infty)$}, but the precise location of the vertices depends on $p$. We refer to this as final layer geometry.

If $h:\R^d\to \R^k$ is given by $h(x) = A\,f(x)$ for $f\in L^\infty(\P;\R^m)$ and a linear map $A:\R^m\to \R^k$, the following holds: If $\|A\|_{L(\ell^2,\ell^2)}\leq 1$ and $\|f(x)\|_{\ell^2}\leq R$ for all $x\in \R^d$, then any energy minimizer satisfies $f(x) = y_i$ for all $x\in C_i$ where the outputs $y_i$ form the vertices of a regular \sw{standard simplex} in a high-dimensional ambient space.
We refer to this as penultimate layer geometry. We note that similar results were obtained in a different framework in \cite{jianfeng_collapse}.

\item Considerations on the final layer geometry are generally independent of the choice of norm on $\R^k$ within the class of $\ell^p$-norms, while the penultimate layer geometry appears to depend specifically on the use of the Euclidean norm. While the coordinate-wise application of a one-dimensional activation function is not hugely compatible with Euclidean geometry (or at least no more compatible than with  $\ell^p$-geometry for any $p\in [1,\infty]$), the transition from the penultimate layer to the final layer is described by a single affine map $y\mapsto Ay +b$. If $A$ and $b$ are initilized from a distribution compatible with Euclidean geometry (e.g.\ a rotation-invariant Gaussian) and optimized by an algorithm such as gradient descent which is based on the Euclidean inner product, then the use of Euclidean geometry for $(A,b)$ is well justified. 

In deeper layers, the significance of Euclidean geometry becomes more questionable. Even for the map $f:\R^d\to\R^m$, it is unclear whether the Euclidean norm captures the constraints on $f$ well. 

\item If $h(x) = \sum_{i=1}^m a_i\,\sigma(w_i^Tx+b_i)$ is a shallow neural network classifier and the weights $(a_i, w_i, b_i)$ are optimized by gradient descent, then in general $h$ does {\em not} converge to a classifier which is constant on different data classes (although the hypothesis class contains functions with arbitrarily low risk which are constant on the different classes $C_i$). This is established in different geometries:
\begin{enumerate}
\item In the first case, $\sigma$ is the ReLU activation function and the classes are linearly separable. Under certain conditions, gradient descent approaches a maximum margin classifier, which can be a linear function and thus generally non-constant over the data classes.
\item In the second case, $\sigma$ is constant for large arguments and there are three data points $x_1, x_2, x_3$ on a line where $x_1, x_3$ belong to the same class, but the middle point $x_2$ belongs to a different class. Then the values of $h$ at $x_1, x_2, x_3$ cannot be chosen independently due to the linear structure of the first layer, and the heuristic behind the toy model does not apply.
\end{enumerate}
\sw{
Note that $h$ is of the form $h=Af$, but $f(x) = \sigma(Wx)$ is not sufficiently expressive for the analysis of the {\em penultimate} layer to apply.
}
\end{enumerate} 

The theoretical analysis raises further questions. As the expressivity of the hypothesis class and the ability to set values on the training set with little interaction between different point evaluations seems crucial to the `neural collapse' phenomenon, we must question whether this simple geometric configuration is in fact desirable, or merely the optimal configuration in a hypothesis class which is too large to allow any statistical generalization bounds. Such concerns were already raised in \cite{elad2020another}. While the latter possibility is suggested by the theoretical analysis, it should be emphasized that in the numerical experiments in \cite{papyan2020prevalence} solutions with good generalization properties are found. This compatibility could be explained by considering a hypothesis class which is not as expressive as $L^\infty(\P;\R^m)$, but contains a function which attains a desired set of values on a realistic data set.

\sw{
It should be noted that the final layer results apply to any sufficiently expressive function class, not just neural networks. The results for the penultimate layer apply to classes of classifiers which are compositions of a linear function and a function in a very expressive function class. In both cases, we consider (norm-constrained) energy minimizers, not training dynamics. If the norm constraints are meaningful for a function model and an optimization algorithm can find the minimizers, the analysis applies in the long time limit, but the dynamics would certainly depend on the precise function model. This coincides with the situation considered by \cite{papyan2020prevalence}, in which the cross-entropy is close to zero after significant training.
}

\sw{
If $h=Af$ and $f$ is not sufficiently expressive (as in two-layer neural networks), we observe that classifier collapse does not occur, even in the final layer. Whether there are further causes driving classifier collapse in deep neural networks remains to be seen.
}

We believe that further investigation in this direction is needed to understand the following: Is neural collapse observed on random data sets or real data sets with randomly permuted labels? Does it occur also on test data or just training data? Is neural collapse observed for ReLU activation functions, or only for activation functions which tend to a limit at positive and negative infinity? Do the outputs over different classes $y_i$ attain a regular simplex configuration also if the weights of the different data classes are vastly different? Is neural collapse observed if a parameter optimization algorithm is used which does not respect Euclidean geometry (e.g.\ an algorithm with coordinatewise learning rates such as ADAM)? 
 The question when neural collapse occurs and whether it helps generalization in deep learning remains fairly open.

The article is structured as follows. In Section \ref{section preliminaries}, we rigorously introduce the problem we will be studying and obtain some first properties. In Sections \ref{section final layer} and \ref{section penultimate layer}, we study a toy model for the geometry of the {output layer} and penultimate layer of a neural network classifier respectively. In Section \ref{section caveats}, we present analytic examples in simple situations where neural network classifiers behave markedly differently and where the toy model analysis does not apply. 

\subsection{Notation}

We consider classifiers $h:\R^d\to\R^k$ in a hypothesis class $\H$. Often, $h$ will be assumed to be a general function on a finite set with norm-bounded output, or the composition of such a function $f:\R^d\to \R^m$ and a linear map $A:\R^m\to \R^k$ for some $m\geq 1$. Variables in $\R^d, \R^m$ and $\R^k$ are denoted by $x, y$ and $z$ respectively.

\section{Preliminaries}\label{section preliminaries}

\subsection{Set-up}

A classification problem is made up of the following ingredients:

\begin{enumerate}
\item A {\em data distribution}, i.e.\ a probability measure $\P$ on $\R^d$.
\item A {\em label function}, i.e.\ a $\P$-measurable function $\xi:\R^d\to \{e_1, \dots, e_k\} \subset \R^k$. We refer to the sets $C_i = \xi^{-1}(\{e_i\})$ as the classes.
\item A {\em hypothesis class}, i.e.\ a class $\H$ of functions $h:\R^d \to \R^k$ for $d\gg 1$ and $k\geq 2$. 
\item A {\em loss function} $\ell:\R^k\times \R^k\to [0,\infty)$.
\end{enumerate}

\sw{We always assume that $\H\subseteq L^1(\P;\R^k)$ and often even $\H\subseteq L^\infty(\P;\R^k)$.} These ingredients are combined in the {\em risk functional}
\begin{equation}
\Risk:\H\to [0, \infty), \qquad \Risk(h) = \int_{\R^d} \ell\big(h(x), \xi_x\big)\,\P(\d x),
\end{equation}
which is approximated by the {\em empirical risk functional}
\[
\widehat \Risk_n(h) = \frac1n \sum_{i=1}^n \ell\big(h(x_i), \xi_i\big)
\]
where $x_i$ are samples drawn from the distribution $\P$ and $\xi_i = \xi_{x_i}$. Since we can write
\[
\widehat \Risk_n(h) = \int_{\R^d} \ell\big(h(x), \xi_x\big)\,\P_n(\d x), \qquad \P_n = \frac1n \sum_{i=1}^n \delta_{x_i},
\]
we do not differentiate between empirical risk and (population) risk in this article. This allows us to organically incorporate that all results are independent of the number of data points. We focus on the {\em softmax cross entropy risk functional} associated to the loss function
\begin{equation}
\ell\big(h, y\big) = - \log \left(\frac{\exp (h\cdot y)}{\sum_{i=1}^k \exp(h\cdot e_k)}\right).
\end{equation}
This loss function allows the following probabilistic interpretation: For given a given classifier $h\in \H$ and data point $x\in \R^d$, the vector $\pi$ with entries
\[
\pi_i(x):= \frac{\exp (h(x)\cdot e_i)}{\sum_{j=1}^k \exp(h(x)\cdot e_j)}
\]
is a counting density on the set of labels $\{1,\dots, k\}$, depending on the input $x$. The function
\[
\Phi:\R^k\to \R^k, \qquad \Phi(h) = \left(\frac{\exp (h\cdot e_1)}{\sum_{i=1}^k \exp(h\cdot e_i)}, \dots, \frac{\exp (h\cdot e_k)}{\sum_{i=1}^k \exp(h\cdot e_i)}\right)
\]
which converts a $k$-dimensional vector into a counting density is referred to as the softmax function since it approximates the maximum coordinate function of $h$ for large inputs. The cross-entropy (Kullback-Leibler divergence) of this distribution with respect to the distribution $\bar \pi(x)$ which gives the correct label with probability $1$ is precisely 
\[
-\sum_{j=1}^k \log \left(\frac{\pi_j(x)}{\bar\pi_j(x)} \right)\,\bar \pi(x) = -\log\left(\frac{\pi_{i(x)}}{1}\right)\cdot 1 = -\log\left(\frac{\exp (h(x)\cdot \xi_x)}{\sum_{i=1}^k \exp(h(x)\cdot e_k)}\right)
\]
since $\bar\pi_{j} = \delta_{j,i(x)}$ and $0\cdot \log(\infty) = 0$ in this case by approximation. The risk functional thus is the average integral of the pointwise cross-entropy of the softmax counting densities with respect to the true underlying distribution. 

Note the following: $\ell>0$, but if $h$ is such that \sw{$h(x)\cdot \xi_x > \max_{e_1, \dots, e_k \neq \xi_x} h(x) \cdot e_i$ for $\P$-almost every $x$}, then 
\[
\lim_{\lambda\to \infty} \Risk(\lambda h) = \lim_{\lambda\to\infty} -\int_{\R^d} \log \left(\frac{\exp (\lambda h(x)\cdot \xi_x)}{\sum_{i=1}^k \exp(\lambda h(x)\cdot e_k)}\right)\,\P(\d x) = 0.
\]
Thus the cross-entropy functional does not have minimizers in suitably expressive function classes which are cones (i.e.\ $f\in \H, \lambda>0 \Ra \lambda f\in \H$). So to obtain meaningful results by energy minimization, we must consider

\begin{enumerate}
\item a dynamical argument concerning a specific optimization algorithm, or
\item a restricted hypothesis class with meaningful norm bounds, or
\item a higher order expansion of the risk.
\end{enumerate}

We follow the first line of inquiry for shallow neural networks in Section \ref{section caveats} and the second line of inquiry for toy models for deep networks in Sections \ref{section final layer} and \ref{section penultimate layer}.

\subsection{Convexity of the loss function} For the following, we note that the softmax cross entropy loss function has the following convexity property.

\begin{lemma}\label{lemma convex loss function}
The function 
\[
\Phi_j:\R^k\to \R, \qquad \Phi(z) = - \log\left( \frac{\exp(z_j)}{\sum_{i=1}^k\exp(z_i)}\right) = \log\left(\sum_{i=1}^k \exp(z_i)\right)-z_j
\]
is convex for any $1\leq j\leq k$ and strictly convex on hyperplanes $H_\alpha$ of the form
\[
H_\alpha = \left\{z\in \R^k : \sum_{j=1}^k z_j = \alpha\right\}.
\]
\end{lemma}

For the sake of completeness, we provide a proof in the Appendix.
Since $\Phi\big(z+ \lambda (1,\dots,1)\big) = \Phi(z)$ for all $\lambda\in \R$, we note that $\Phi_j$ is {\em not} strictly convex on the whole space $\R^d$.

\section{Heuristic geometry: final layer}\label{section final layer}

\subsection{Collapse to a point}

In this section, we argue that the output $h(C_i)$ of the classifier should be a single point for all classes $C_i$, $i=1,\dots,k$ if the hypothesis class is sufficiently expressive. We will discuss the penultimate layer below.

\begin{lemma}\label{lemma tetrahedron}
Let $h\in \H$ and set
\[
z_i:= \frac1{|C_i|}\int_{C_i} h(x')\,\P(\d x'), \qquad \bar h(x) = z_i \quad\text{for all }x\in C_i.
\]
Then $\Risk(\bar h) \leq \Risk(h)$ \sw{and equality holds if and only if there exists a function $\lambda\in L^1(\P)$ such that $h-\bar h = \lambda (1,\dots,1)$ $\P$-almost everywhere}. 
\end{lemma}

The reasoning behind the Lemma is that
\begin{align*}
\int_{C_i} \Phi_i(h(x))\,\P(\d x) &\approx \int_{C_i} \Phi_i\left(z_i \right) + \nabla \Phi_i(z_i) \cdot \big(h(x) - z_i\big) + \frac12 \big(h(x) - z_i\big)^T D^2\Phi_i(z_i)\,\big(h(x) - z_i\big) \,\P(\d x)\\
	&= \int_{C_i} \Phi_i\left(z_i \right) \,\P(\d x) + \nabla \Phi_i(z_i) \cdot \int_{C_i} h(x) - z_i\,\P(\d x)\\
	&\qquad + \frac12 \int_{C_i} \big(h(x) - z_i\big)^T D^2\Phi_i(z_i)\,\big(h(x) - z_i\big) \,\P(\d x)\\
	&\geq \int_{C_i} \Phi_i\left(z_i \right) \,\P(\d x)
\end{align*}
since the first order term vanishes. A summation over $i$ establishes the result. A rigorous proof using Jensen's inequality can be found in the appendix.

Thus if a class $C_j$ is mapped to a set $h(C_j)\subseteq \R^k$ with a prescribed center of mass, different classes are mapped to the same centers of mass, it is energetically favorable to reduce the variance to the point that $h(C_j)$ is a single point. Whether or not this is attainable depends primarily on the hypothesis class $\H$, but a very expressive class like deep neural networks is likely to allow this collapse to a single point. 

\begin{corollary}
If \sw{$\H = L^\infty(\P;V)$ is the class of bounded $\P$-measurable functions which take values in a compact convex} set $V\subset\R^k$, then \sw{a} minimizer $h$ or $\Risk$ in $\H$ \sw{can be taken to map} the class $C_i$ to a single point $z_i\in V$ for all $i=1,\dots, k$, \sw{and all other minimizer differ from $h$ only in direction $(1,\dots,1)$.}
\end{corollary}

\subsection{\sw{Simplex} configuration}

In this section, we discuss the emergence of the simplex configuration under the assumption that the every class gets mapped to a single point $z_i\in \R^k$, or equivalently that each class consists of a single data point. Again, we consider the {\em last layer problem}: Assume that 
\begin{itemize}
\item $\mathcal X=\{x_1,\dots, x_d\}$,
\item $\H$ is the class of functions from $\mathcal X$ to the Euclidean ball $B_R(0)$ in $\R^k$. 
\end{itemize}
 Let $\P$ be a probability measure on $\mathcal X$ and $p_i:= \P(\{x_i\})$. We wish to solve the minimization problem $h^*\in \argmin_{h\in \H} \Risk(h)$ where
\begin{align*}
\Risk(h) = \int_\X -\log\left(\frac{\exp(h(x)\cdot \xi_x)}{\sum_{i=1}^d \exp(h(x)\cdot e_i)}\right)\,\P(\d x) = -\sum_{i=1}^d p_i \log\left(\frac{\exp(h(x_i)\cdot e_i)}{\sum_{j=1}^d\exp(h(x_i)\cdot e_j)}\right).
\end{align*}
Due to our choice of hypothesis class, there is no interaction between $h(x_i)$ and $h(x_j)$, so we can minimize the sum term by term:
\[
z_i:= h(x_i) \in \argmin_{z\in B_R(b)} \left(- \log\left(\frac{\exp(z\cdot e_i)}{\sum_{j=1}^d \exp(z\cdot e_j)}\right)\right) = \min_{z\in B_R(b)}\Phi_i(z)
\]
where $\Phi_i(z) = \log\left(\sum_{j=1}^k \exp(z_k)\right) - z_i$ is as in Lemma \ref{lemma convex loss function}.

\begin{lemma}\label{lemma final tetrahedron}
For every $i$ there exists a unique minimizer $z_i$ of $\Phi_i$ in $B_R(0)$ and $z_i = \alpha\,e_i + \beta\sum_{j\neq i} e_j$ for $\alpha, \beta\in \R$ which do not depend on $i$.
\end{lemma}

Since $\Phi\big(z+ \lambda(1,\dots,1)\big) = \Phi(z)$ for all $\lambda\in\R$, the same result holds for the ball $B_R\big(\lambda(1,\dots, 1)\big)$ with any $\lambda\in\R$.
We can determine the minimizers by exploiting the relationships
\[
\alpha^2 + (k-1)\beta^2 = R^2, \qquad \alpha + (k-1)\beta = 0
\]
which are obtained from the Lagrange-multiplier equation \eqref{eq lagrange multiplier} in the proof of Lemma \ref{lemma final tetrahedron}. The equations reduce to
\[
\alpha = (k-1)\,\sqrt{\frac{R^2-\alpha^2}{k-1}} = \sqrt{(k-1)\,(R^2-\alpha^2)}\quad\Ra\quad \alpha^2 = (k-1)\,(R^2 - \alpha^2)
\]
and ultimately
\begin{equation}\label{eq alpha, beta}
 \alpha^2 = \frac{k-1}k\,R^2 \quad\Ra\quad \alpha = \sqrt{\frac{k-1}k}\,R, \quad \beta = -\frac{1}{k-1}\alpha = -\frac{R}{\sqrt{k(k-1)}}.
\end{equation}

\begin{remark}
Lemma \ref{lemma tetrahedron} remains true when $B_R(0)$ is the ball of radius $R>0$ with respect to an $\ell^p$-norm on $\R^k$ for $1<p < \infty$ (with different values for $\alpha$ and $\beta$) -- see appendix for further details.
\end{remark}

\begin{corollary}
If $\H$ is \sw{the unit ball in $L^\infty(\P; \R^k)$ where $\R^k$ is equipped with the $\ell^p$-norm for $1<p<\infty$}, then any minimizer $h$ of $\Risk$ in $\H$ satisfies that $h(C_i)$ is a single point $z_i$ for all $i=1,\dots, k$ and the points $z_i$ form the vertices of a regular \sw{standard simplex}.
\end{corollary}

\begin{remark}
A major simplification in our analysis was the restriction to one-point classes and general functions on the finite collection of points or more generally to \sw{bounded} $\P$-measurable functions. In other hypothesis classes, the point values $h(x_i)$ and $h(x_j)$ cannot be chosen independently. It is therefore no longer possible to minimize all terms in the sum individually, and trade-offs are expected. In particular, while our analysis was independent of the weight $p_i= \P(C_i)$ of the individual classes, these are expected to influence trade-offs in real applications. 

Nevertheless, we record that simplex configurations are favored for hypothesis classes $\H$ with the following two properties:
\begin{enumerate}
\item $\H$ is expressive enough to collapse classes to single points and to choose the values on different classes almost independently, and
\item functions in $\H$ respect the geometry of $\R^k$ equipped with an $\ell^p$-norm in a suitable manner.
\end{enumerate}
\end{remark}

\section{Heuristic geometry: penultimate layer}\label{section penultimate layer}

Above, we obtained rigorous results for the final layer geometry under heuristic assumptions. In this section, we consider a hypothesis class $\H$ in which functions can be decomposed as
\[
h_{f, A, b}(x) = A f(x) + b\qquad\text{where }f:\R^d\to \R^m, \quad A:\R^m\to \R^k, \quad b\in \R^k
\]
and we are interested in the geometry of $f$ and $A$. Typically, we imagine the case that $m\gg k$.

\subsection{Collapse to a point}

We have given a heuristic proof above that it is energetically favorable to contract $h(C_i)$ to a single point $z_i\in \R^k$ under certain conditions. Since $A:\R^m\to \R^k$ has a non-trivial kernel for $m>k$, this is a weaker statement than claiming that $f$ maps $C_i$ to a single point $y_i \in \R^m$. We note the following: $V_i = (A\cdot + b)^{-1}(z_i)$ is an $m-k$-dimensional affine subspace of $\R^m$. In particular, a strictly convex norm (e.g.\ an $\ell^p$-norm for $1<p<\infty$) has a unique minimum $y_i\in V_i$. Thus if we subscribe to the idea that $f$ is constrained by an $\ell^p$-norm, it is favorable for $f$ to collapse $C_i$ to a single point $y_i\in \R^m$.

Heuristically, this situation arises either if it is more expensive to increase the norm of $f$ than change its direction, or if $(A,b)$ evolve during training and it is desirable to bring $f(x)$ towards the minimum norm element of $(A\cdot b)^{-1}(z_i)$ to increase the stability of training. The first consideration applies when $A, b$ are fixed while the second relies on the variability of $(A,b)$. Their relative importance could therefore be assessed numerically by initializing the final layer variables in a simplex configuration and making them non-trainable.

If $\sigma$ is a bounded activation function, the direction of the final layer output depends on the coefficients of all layers in a complicated fashion, while its magnitude \sw{mostly} depends on the final layer coefficients. \sw{We can imagine gradient flows as continuous time versions of the minimizing movements scheme
\[
\theta_{n+1} \in \argmin_\theta \frac1{2\eta}\,\|\theta_n - \theta\|^2 + \Risk\big(h(\theta_n,\cdot)\big)
\]
where $h(\theta,\cdot)$ is a parameterized function model. Using the unweighted Euclidean norm for the gradient flow, we allow the same budget to adjust final layer and deep layer coefficients. It may therefore be easier to adjust the direction of the output than the norm. 
} For ReLU activation on the other hand, the magnitude of the coefficients in all layers combines to an output in a multiplicative fashion. \sw{It may well be that neural collapse is more likely to occur} for activation functions which tend to a finite limit at positive and negative infinity.

In section \ref{section counterexample three neurons}, we present examples which demonstrates that if all data points are not collapsed to a single point in the penultimate layer, they may not collapse to a single point in the final layer either when the weights of a neural network are trained by gradient descent. This is established in two different geometries for different activation functions


\subsection{Simplex configuration}

We showed above that {\em any} $\ell^p$-geometry leads to simplex configurations in the last layer for certain toy models. When considering the geometry of the penultimate layer, we specifically consider $\ell^2$-geometry. This is justified for $A, b$ since the parameters are typically initialized according to a normal distribution (which is invariant under general rotations) and optimized by (stochastic) gradient descent, an algorithm based on the Euclidean inner product. For compatibility purposes, also the output of the preceding layers $f$ should be governed by Euclidean geometry. 

Again, as a toy model we consider the case of one-point classes. To simplify the problem, we furthermore suppress the bias vector of the last layer. Let
\begin{enumerate}
\item $\X = \{x_1, \dots, x_k\}\subset \R^d$,
\item $f:\X\to B_R(0)\subseteq\R^m$, and
\item $A:\R^m\to\R^k$ linear.
\end{enumerate}
As before $B_R(0)$ denotes the Euclidean ball of radius $R>0$ centered at the origin in $\R^m$. We denote $h(x) = A f(x)$, $y_i := f(x_i)\in \R^m$ and $z_i:= h(x_i)\in \R^k$. As we suppressed the bias of the last layer, we could normalize the center of mass in the penultimate layer to be $\frac1k \sum_{i=1}^k y_i =0$. Instead, we make the (weaker) assumption that $y_i \in B_R(0)$ for some $R>0$ and all $i=1,\dots, k$.

We assume that the outputs $h(x_i)$ are in the optimal positions in the last layer and show that if $A$ has minimal norm, also the outputs $f(x_i)$ in the penultimate layer are located at the vertices of a regular \sw{standard simplex}.
Denote by
\[
\|A\|_{L(\ell^2,\ell^2)} = \max_{\|x\|_{\ell^2}\leq 1} \frac{\|Ax\|_{\ell^2}}{\|x\|_{\ell^2}}
\]
the operator norm of the linear map $A$ with respect to the Euclidean norm on both domain and range.

\begin{lemma}\label{lemma penultimate tetrahedron}
Let $m\geq k-1$ and $y_i\in B_R(0)\subseteq \R^{m}$ and $A:\R^{m}\to \R^k$ linear such that $Ay_i = z_i$ where $z_i$ are the vertices of the regular \sw{standard simplex} described in Lemma \ref{lemma tetrahedron} and \eqref{eq alpha, beta}. Then 
\begin{enumerate}
\item the center of mass of outputs $y_i$ of $f$ is $\frac1k \sum_{i=1}^k y_i =0$,
\item $\|A\|_{L(\ell^2,\ell^2)} \geq 1$, and 
\item $\|A\|_{L(\ell^2, \ell^2)}=1$ if and only if
\begin{enumerate}
\item $A$ is an isometric embedding of the $k-1$-dimensional subspace spanned by $\{y_1,\dots, y_k\}$ into $\R^k$ and 
\item $y_i$ are vertices of a regular \sw{standard simplex} with the same side lengths.
\end{enumerate}
\end{enumerate}
\end{lemma}

The proof is given in the appendix. We conclude the following.

\begin{corollary}
For any $m\geq k-1$, consider the hypothesis class
\[
\H = \left\{ h:\R^d\to \R^k \:\bigg|\: h = Af \:\text{where } \begin{array}{ll}f:\R^d\to \R^m \text{ is }\P-\text{measurable}, &\|f(x)\|_{\ell^2}\leq R\:\:\:\P-\text{a.e.}\\ \: A:\R^m\to \R^k \text{ is linear}, &\|A\|_{L(\ell^2,\ell^2)}\leq 1\end{array}\right\}.
\]
Then a minimizer $h\in \H$ of $\Risk$ satisfies $h= Af$ where
\begin{enumerate}
\item there exist values $y_i \in \R^m$ such that $f(x) = y_i$ for almost every $x\in C_i$,
\item the points $y_i$ are located at the vertices of a regular $k-1$-dimensional \sw{standard simplex} in $\R^m$,
\item the center of mass of the points $y_i$ (with respect to the uniform distribution) is at the origin, and
\item $A$ is an isometric embedding of the $k-1$-dimensional space spanned by $\{y_1,\dots, y_k\}$ into $\R^k$.
\end{enumerate}
\end{corollary}

\begin{remark}
The restriction to the Euclidean case is because in Euclidean geometry, any $k-1$-dimensional subspace of $\R^d$ is equipped with the Euclidean norm in a natural way. For other $\ell^p$-spaces, the restriction of the $\ell^p$-norm is not a norm of $\ell^q$-type and we cannot apply Lemma \ref{lemma final tetrahedron}.
\end{remark}

Thus, we conclude that a simplex geometry is desirable also in the penultimate layer of a function $h(x) = A f(x)$ if
\begin{enumerate}
\item the function class $\mathcal F$ in which $f$ is chosen and the linear matrix class in which $A$ is chosen respect the Euclidean geometry of $\R^m$,
\item $\F$ is sufficiently expressive to collapse all data points in the class $C_i$ to a single point $y_i$ and
\item $\F$ is so expressive that $y_i$ and $y_j$ can be chosen mostly independently.
\end{enumerate}

\section{Caveats: Binary classification using two-layer neural networks}\label{section caveats}

In this section we consider simple neural network classifier models and data sets on which we can show that the classes {\em are not} collapsed into single points when the model parameters are trained by gradient descent, despite the fact that the function class is sufficiently expressive. This is intended as a complementary illustration that the heuristic considerations of Sections \ref{section final layer} and \ref{section penultimate layer} may or may not be valid, depending on factors which are yet to be understood.

Deep neural networks with many nonlinearities can be a lot more \sw{flexible} than shallow neural networks, and the intuition we built up above does not quite apply here. However, we emphasize that a deep neural network $h$ can be decomposed as $h = g\circ f$ where $f:\R^d\to \R^k$ is a deep neural network and $g:\R^k\to \R$ is a shallow neural network. All results should therefore be considered also valid in deep classification models where only the outermost two layers are trained. This is a more realistic assumption in applications where large pretrained models are used to preprocess data and only the final layers are trained for a specific new task. Similarly, we note that this indicates that if data is non-collapsed two layers before the output, then it may not collapse in the output layer either.

The examples we consider concern {\em binary} classification, i.e.\ all functions take values in $\R$ rather than a higher-dimensional space. The label function $x\mapsto \xi_x$ takes values in $\{-1,1\}$ instead of the set of basis vectors.
For the sake of convenience, the data below are assumed to be one-dimensional, but similar results are expected to hold when data in a high-dimensional space is either concentrated on a line or classification only depends on the projection to a line.

\subsection{Two-layer ReLU-networks in the mean field scaling}

Consider the {\em mean field scaling} of shallow neural networks, where a network function is described as
\[
f(x) = \frac1m\sum_{i=1}^m a_i\,\sigma(w_i^Tx+b_i)\qquad \text{rather than}\quad f(x) = \sum_{i=1}^m a_i\,\sigma(w_i^Tx+b_i).
\]
In this regime, it is easy to take the infinite width limit
\begin{equation}\label{eq h pi}
f(x) = \int_{\R^k\times \R^d\times \R} a\,\sigma(w^Tx+b)\,\pi(\d a \otimes \d w\otimes \d b)
\end{equation}
with general weight distributions $\pi$ on $\R^{k+d+1}$. We denote the functions as represented in \eqref{eq h pi} by $h_\pi$. Finite neural networks are a special case in these considerations with distribution $\pi = \frac1m \sum_{i=1}^m \delta_{(a_i, w_i, b_i)}$. We recall the following results.

\begin{proposition}\cite{chizat2018global}
All weights $(a_i, w_i, b_i)$ evolve by the gradient flow of
\[
(a_i, w_i, b_i)_{i=1}^m \mapsto \Risk\left(\frac1m\sum_{i=1}^m a_i\,\sigma(w_i^Tx+b_i)\right)
\]
 in $(\R^{k+d+1})^m$ if and only if the empirical distribution $\pi = \frac1m \sum_{i=1}^m \delta_{(a_i, w_i, b_i)}$ evolves by the Wasserstein gradient flow of
 \begin{equation}\label{eq risk mean field}
 \pi\mapsto \Risk\left(h_\pi\right)
\end{equation}
(up to time rescaling).
\end{proposition}

Consider specifically $\sigma(z) = \max\{z,0\}$ and $k=1$ with the risk functional
\[
\Risk(h) = -\int_{\R^d}\log\left( \frac{\exp(- h(x)\cdot \xi_x)}{\exp(h(x)) + \exp(-h(x))}\right)\,\P(\d x).
\]
The following result applies specifically to the Wasserstein gradient flow of certain continuous distributions, which can be approximated by finite sets of weights.

\begin{proposition}\cite{Chizat:2020aa}
Assume that $\pi^0$ is such that $|a|^2 \leq |w|^2 + |b|^2$ almost surely and such that
\[
\pi^0\left(\{(w,b) \in \Theta\}\right) >0
\]
for every open cone $\Theta$ in $\R^{d+1}$. Let $\pi^t$ evolve by the Wasserstein gradient flow of \eqref{eq risk mean field} with initial condition $\pi^0$. Then (under additional technical conditions), the following hold:
\begin{enumerate}
\item $\xi_x\,h_{\pi^t}(x) \to +\infty$ for $\P$-almost every $x$.
\item There exist
\begin{equation}\label{eq max margin definition}
\pi_* \in \argmax \left\{\min_{x\in \spt\,\P} \big(\xi_x\cdot h_\pi(x)\big) \:\bigg|\:\pi\text{ s.t. }\int_{\R^{d+2}}|a|\,\big[|w|+|b|\big]\,\d\pi \leq 1\right\}
\end{equation}
and a normalizing function $\mu:[0, \infty)\to (0,\infty)$ such that $\mu(t)\,h_{\pi_t}\to h_{\pi^*}$ locally uniformly on $\R^d$. 
\end{enumerate}
\end{proposition}

\begin{remark}
We call $h^*$ the {\em maximum margin classifier} in Barron space.
Both the normalization condition in \eqref{eq max margin definition} and the normalizing function $\mu$ are related to the Barron norm or variation norm of classifier functions. The existence of a minimizer in \eqref{eq max margin definition} is guaranteed by compactness. Existence of a limit of $\pi^t$ in some weak sense has to be assumed a priori in \cite{chizat2018global}.
\end{remark}

\begin{remark}
The open cone condition is satisfied for example if $\pi_0$ is a normal distribution on $\R^{d+1}$, which is a realistic distribution. This property ensures a diversity in the initial distribution, which is required to guarantee convergence. The smallness condition on $a$ is purely technical and required to deal with the non-differentiability of the ReLU activation function, see also \cite{relutraining}. The same result holds without modification for leaky-ReLU activation. With some additional modifications, it is assumed to also extend to smooth and bounded activation functions.
\end{remark}

\begin{remark}
The divergence $\xi_x\,h_{\pi^t}(x) \to +\infty$ is expected to be logarithmic in time, which can almost be considered bounded in practice. The convergence $h_{\pi^t}\to h^*$ is purely qualitative, without a rate.
\end{remark}

Consider a binary classification problem in $\R$ where $C_{-1}= [-2, -1]$ and $C_1=[1,2]$. 

\begin{lemma}\label{lemma max margin}
Consider a binary classification problem in $\R$ where one class $C_{-1}$ with label $\xi = -1$ is contained in $[-2,-1]$ and the other class $C_1$ with label $\xi=+1$ is contained in $[1,2]$. Assume that $-1\in C_{-1}, 1\in C_1$ and that both classes contain at least one additional point. 

The classification problem admits a continuum of maximum margin classifiers 
\[
f_b(x) = \frac{1}{2\,[1+b]}\begin{cases}{x+b}&x>b\\ 2x& -b<x<b\\{x-b} & x< -b\end{cases}
\]
parametrized by $b\in[0,1]$.
\end{lemma}

In particular, we expect that $h_{\pi_t}$ is not constant on either of the classes $[1,2]$ or $[-2, -1]$. The proof is postponed until the appendix.

\begin{remark}
We described the mean field setting in its natural scaling. However, the same results are true (with a different time rescaling) if $f$ is represented in the usual fashion as $f(x) = \sum_{i=1}^m a_i\,\sigma(w_i^Tx+b_i)$ without the normalizing factor $\frac1m$, assuming that the weights are initialized such that $a_i, w_i, b_i \sim m^{-1/2}$.
\end{remark}

\subsection{Two-layer networks with non-convex input classes}\label{section counterexample three neurons}

Assume that
\[
\P = p_{1}\,\delta_{-1} + p_2 \,\delta_0 + p_3\,\delta_1, \qquad \sw{p_1, p_2, p_3 \geq 0,} \qquad p_{1} + p_2+ p_3 = 1
\]
and that $\xi_{-1} = \xi_1 = 1$ and $\xi_0 = -1$. We consider the risk functional
\[
\Risk(h) = \int_\R \exp\big(- \xi_x h(x)\big)\,\P(\d x) = p_{1}\,\exp\big(- h(-1)\big) + p_2\,\exp\big( h(0)\big) + p_{3}\,\exp\big(-h(1)\big),
\]
which is similar to cross-entropy loss in its tails since
\begin{align*}
-\log\left(\frac{\exp(\xi_x\,h(x))}{\exp(\xi_x\,h(x)) + \exp(-\xi_x\,h(x))}\right) &= -\log\left(\frac1{1+ \exp(-2\,\xi_xh(x))}\right)\\
	& \approx 1 - \frac1{1+ \exp(-2\,\xi_xh(x))}\\
	&= \frac{\exp(-2\,\xi_xh(x))}{1+ \exp(-2\,\xi_xh(x))} \\
	&\approx \exp(-2\,\xi_xh(x))
\end{align*}
if $\xi_x\,h(x)$ is large. Further assume that the classifier is a shallow neural network with three neurons
\[
h(x) = \sum_{i=1}^3 a_i\,\sigma(w_i x+b_i).
\]
To make life easier, we consider a simplified sigmoid activation function $\sigma:\R\to\R$ which satisfies $\sigma(z) = 0$ for $z\leq 0$ and $\sigma(z) = 1$ for $z\geq 1$, and we assume that the parameters $(a_i, w_i, b_i)$ are initialized such that
\begin{equation}\label{eq three neuron network}
h(x) = a_1\,\sigma(-x) - a_2\sigma(x+1) + a_3\,\sigma(x)
\end{equation}
In particular, $\sigma'(w_ix+b_i) =0$ for $\P$-almost every $x$ at initialization and all $i=1,2,3$. This implies that $(w_i, b_i)$ are constant along gradient descent training, so only $a_1, a_2, a_3$ evolve. We can write
\[
\Risk\big(  - a_1\,\sigma(x) + a_2\sigma(x+1) - a_3\,\sigma(x-1)\big) = p_1\,\exp(-a_1) + p_2\,\exp(-a_2) + p_3\,\exp(a_2-a_3). 
\]

\begin{lemma}\label{lemma convergence three neurons}
Let $h = h_{a_1, a_2, a_3}$ be as in \eqref{eq three neuron network} for $a_1, a_2, a_3\in \R$. Assume that $a_1, a_2, a_3$ evolve by the gradient flow of $F(a_1, a_2, a_3) = \Risk(h_{a_1,a_2,a_3})$. Then
\[
\lim_{t\to\infty} \big[h(t, 1) - h(t, -1)\big] = 0 \qquad\LRa\qquad p_3 = 2p_1
\]
independently of the initial condition $(a_1, a_2, a_3)(0)$.
\end{lemma}

In general, assume that $h = f\circ g$ where $f$ is a shallow neural network. Assume that there are two classes $C_i, C_j$ such that the convex hull of $g(C_i)$ intersects $g(C_j)$. Then it is questionable that classes can collapse to a single point in the final layer. While this does not imply that $g(C_i)$ and $g(C_j)$ should concentrate around the vertices of a regular \sw{standard simplex}, it suggests that simple geometries are preferred already {\em before} the penultimate layer if the $h$ is to collapse $C_i$ to a single point.

The proof of Lemma \ref{lemma convergence three neurons} is given in the appendix.

\begin{remark}
We note that the probabilities of the different data points crucially enter the analysis, while considerations above in Lemma \ref{lemma final tetrahedron} were entirely independent of the weight of different classes. The toy model does not capture interactions between the  function values at different data points, which is precisely what drives the dynamics here. 
\end{remark}
\bibliographystyle{../../../alphaabbr}
\bibliography{../../../NN_bibliography}
\newpage

\appendix

\section{Proofs}

\subsection{Proof from Section \ref{section preliminaries}}

We prove the convexity property of the loss function.

\begin{proof}[Proof of Lemma \ref{lemma convex loss function}]
Without loss of generality $j=1$ and we abbreviate $\Phi= \Phi_1$. We compute 
\begin{align*}
\nabla \Phi(z) &= -e_1 + \sum_{j=1}^k \frac{\exp(z_j)}{\sum_{i=1}^k \exp(z_i)}\,e_j\\
\partial_j\partial_l\Phi(z) &= \frac{\exp(z_j)}{\sum_{i=1}^k \exp(z_i)}\,\delta_{jl} - \frac{\exp(z_j)\exp(z_l)}{\left(\sum_{i=1}^k \exp(z_i)\right)^2}\\
	&= p_j\,\delta_{jl} - p_jp_l
\end{align*}
where $p_j = \frac{\exp(z_j)}{\sum_{i=1}^k \exp(z_i)}$. Thus
\begin{align*}
a^t D^2\Phi\,a &= \sum_{i=1}^k a_i^2p_i - \sum_{i,j=1}^k a_i a_j p_i p_j\\
	&= \sum_{i=1}^k a_i^2p_i - \left(\sum_{i=1}^k a_i p_i\right)^2\\
	&= \|a\|_{\ell^2(p)}^2 - \|a\|_{\ell^1(p)}^2\\
	&\geq 0
\end{align*}
since $p$ is a counting density on $\{1,\dots, k\}$. Since $p$ is a vector with strictly positive entries, equality is attained if and only if $a$ is a multiple of $(1,\dots,1)$.
 Since the Hessian of $\Phi$ is positive semi-definite, the function is convex.
\end{proof}

\subsection{Proofs from Section \ref{section final layer}}

The rigorous proof that it is advantageous to collapse the output of a classifier to the center of mass over a class goes as follows.

\begin{proof}[Proof of Lemma \ref{lemma tetrahedron}]
\sw{Denote by $P^\sharp$ the orthogonal projection of $h$ onto the orthogonal complement of the space spanned by the vector $(1,\dots,1)$ and observe that $\ell(P^\sharp h, \xi) = \ell(h, \xi)$ for all $h,\xi$.}

We compute by the vector-valued Jensen's inequality that
\begin{align*}
\Risk(h) &= \Risk(P^\sharp h)\\
	&= - \int_{\R^d} \log\left(\frac{\exp(\sw{P^\sharp}h(x)\cdot \xi_x)}{\sum_{i=1}^k \exp(\sw{P^\sharp}h(x)\cdot e_i)}\right)\P(\d x)\\
	&= - \sum_{j=1}^k \int_{C_j} \log\left(\frac{\exp(\sw{P^\sharp}h(x)\cdot e_j)}{\sum_{i=1}^k \exp(\sw{P^\sharp}h(x)\cdot e_i)}\right)\P(\d x)\\
	&= \sum_{j=1}^k |C_j| \,\frac1{|C_j|} \int_{C_j} \Phi_j(\sw{P^\sharp}h(x)) \P(\d x)\\
	&\geq \sum_{j=1}^k |C_j|\,\Phi_j\left( \frac1{|C_j|}\int_{C_j}\sw{P^\sharp}h(x)\,\P(\d x)\right)\\
	&=  - \int_{\R^d} \log\left(\frac{\exp\big(\overline{\sw{P^\sharp} h}(x)\cdot \xi_x\big)}{\sum_{i=1}^k \exp\big(\overline{\sw{P^\sharp}h}(x)\cdot e_i\big)}\right)\P(\d x)
\end{align*}
\sw{and note that the inequality is strict unless $\sw{P^\sharp}h(x) = \overline{\sw{P^\sharp}h}(x)$ for $\P$-almost every $x$ since $\ell$ strictly convex on the orthogonal complement of $(1,\dots,1)$. This is the case if and only if $h(x) - \bar h(x) \in \mathrm{span}\{(1,\dots,1)\}$ for almost all $x$.}
\end{proof}

We proceed to show the optimality of a simplex configuration in the toy problem.

\begin{proof}[Proof of Lemma \ref{lemma final tetrahedron}]
{\bf Step 1. Existence of the minimizer.} Due to the convexity of $\Phi_i$ is convex on the compact convex set $\overline{B_R(0)}$, $\Phi_i$ has a minimizer $z_i$ in $\overline{B_R(0)}$.

{\bf Step 2. Uniqueness of the minimizer.} 
By the Lagrange multiplier theorem, there exists $\lambda_i\in \R$ such that
\begin{align*}
0 &= \big(\nabla\Phi_i\big)(z_i) -\lambda z_i\\
	&=\left[ \sum_{j=1}^k\frac{\exp(z_i\cdot e_j)}{\sum_{l=1}^k\exp(z_i\cdot e_l)}\,e_k \right]- e_i - \lambda_i z_i\\
	&= \sum_{j=1}^k \left[\frac{\exp(z_i\cdot e_j)}{\sum_{l=1}^k \exp(z_i\cdot e_l)} - \delta_{ij} - \lambda_i\,(z_i\cdot e_j)\right]e_j.\showlabel\label{eq lagrange multiplier}
\end{align*}
All coefficients in the basis expansion have to vanish separately, so in particular
\[
0 = \sum_{j=1}^k \left[\frac{\exp(z_i\cdot e_j)}{\sum_{l=1}^d \exp(z_i\cdot e_l)} - \delta_{ij} - \lambda_i\,(z_i\cdot e_k)\right] = 1-1 - \lambda \sum_{j=1}^k (z_i\cdot e_j) = -\lambda \sum_{j=1}^k (z_i\cdot e_j),
\]
meaning that either $\lambda_i =0$ or $\sum_{j=1}^k (z_i \cdot e_j) =0$. Since $\frac{\exp(z_i\cdot e_j)}{\sum_{l=1}^k \exp(z_i\cdot e_l)}-\delta_{ij}\neq 0$ for any $i, j$ and choice of $z_i$, we find that $\lambda_i\neq 0$ and thus $z_i \in \partial B_R(0)$ and 
\[
0 = \sum_{i=1}^k (z_i\cdot e_k) = (1,\dots,1) \cdot z_i.
\]
Since $\Phi_i$ is {\em strictly} convex in the hyperplane $H= \{z\in \R^k : (1,\dots, 1)\cdot z=0\}$ by Lemma \ref{lemma convex loss function}, we find that the minimizer $z_i \in B_R(0)\cap H$ is unique.

{\bf Step 3. Symmetry.} Since the minimizer $z_i$ is unique and $\Phi_i(z^1, \dots, z^k)$ is invariant under the permutation of the coordinate entries $z^j$ of its argument for $j\neq i$, we find that also the minimizer $z_i$ must have this invariance, i.e.\ 
\[
z_i = \alpha_i e_i + \beta_i \sum_{j\neq i} e_j.
\]
Using symmetry, we find that $\alpha_i\equiv \alpha, \beta_i\equiv \beta$ independently of $i$.
\end{proof}

\begin{remark}
The first and third step of the proof go through for general $\ell^p$-norms since also these norms are invariant under the rearrangement of coordinates. The second step requires slightly different reasoning. Still, the Lagrange-multiplier equation
\[
0 = \sum_{j=1}^k \left[\frac{\exp(z_i\cdot e_j)}{\sum_{l=1}^k \exp(z_i\cdot e_l)} - \delta_{ij} - \lambda_i\,\big|z_i\cdot e_j\big|^{p-2}\,(z_i\cdot e_j)\right]e_j
\]
can be used to conclude $\lambda_i\neq 0$ and thus that any minimizer $z_i$ must lie in the boundary of $B_R(0)$. Now assume that there are multiple minimizers $z_{i,1}$ and $z_{i,2}$. Then $\Phi_i$ cannot be uniformly convex along the connecting line between $z_{i,1}$ and $z_{i,2}$. Therefore $z_{i,2}-z_{i,1} \parallel (1,\dots,1)$. Since the ball $B_R(0)$ is strictly convex and $\Phi_i$ is constant along the connecting line, this is a contradiction to the fact that the minimum is only attained on the boundary.

The equations which determine $\alpha>0,\beta<0$ become
\[
|\alpha|^p + (k-1)\,|\beta|^p = R^p, \qquad |\alpha|^{p-2}\alpha + (k-1)\,|\beta|^{p-2}\beta = 0
\]
which is solved by
\[
\alpha = \left(\frac{(k-1)^\frac1{p-1}}{1+ (k-1)^\frac1{p-1}}\right)^\frac1p \,R, \qquad \beta = - \left(\frac{1 - \frac{(k-1)^\frac1{p-1}}{1+ (k-1)^\frac1{p-1}}}{k-1}\right)^\frac1p R.
\]
If $p\in\{1,\infty\}$, the unit spheres in $\R^k$ have straight segments and singularities, and the Lagrange-multiplier theorem no longer applies. However, we note that the facets of the $\ell^\infty$-unit ball are never parallel to $(1,\dots,1)$, and that the same statement is expected to hold. The same is true for the $\ell^1$-unit ball close to points of the form $\alpha e_i + \beta \sum_{j\neq i}e_j$ if $k>2$. 
\end{remark}

\subsection{Proofs from Section \ref{section penultimate layer}}

Now we show that the simplex symmetry is optimal under certain conditions.

\begin{proof}[Proof of Lemma \ref{lemma penultimate tetrahedron}]
We have
\[
\|A\|_{\ell^2} = \sup_{\|y\| \leq R} \frac{\|Ay\|}{\|y\|} \geq \max_{1\leq i \leq k} \frac{\|z_i\|}{\|y_i\|} = \max_{1\leq i\leq k} \frac{R}{\|y_i\|}\geq 1.
\]
In particular $\|A\|_{\ell^2}\geq 1$ and if $\|A\|_{\ell^2}=1$, then $\|y_i\|=R$ for all $1\leq i\leq k$. 

We observe that the collection $\{z_1, \dots, z_{k-1}\}$ spans the $k-1$-dimensional hyperplane $H = \{z\in \R^k : (1,\dots,1)\cdot z = 0\}$ in $\R^k$. Consequently, the collection $\{y_1,\dots, y_{k-1}\}$ must be linearly independent in $\R^{m}$, i.e.\ the basis of a $k-1$-dimensional subspace. The map $A$ is therefore injective and uniquely determined by the prescription $z_i = Ay_i$ for $i=1,\dots, k-1$. Since
\[
0 = \sum_{j=1}^k z_j = \sum_{j=1}^k (Ay_j) = A\left(\sum_{j=1}^k y_j\right),
\]
we conclude by injectivity that $\sum_{j=1}^{k} y_j = 0$. After a rotation, we may assume without loss of generality that $m=k-1$. Since rotations are Euclidean isometries, also $\R^{k-1}$ is equipped with the $\ell^2$-norm. Assume that $\|A\|_{\ell^2}=1$. Then 
\begin{enumerate}
\item $\|y_j\|=R$ for all $j=1,\dots,k$ and
\item $\sum_{j=1}^k y_j = 0$.
\end{enumerate}
This implies that for every $i=1,\dots, k$ we have
\[
\sum_{j=1}^k \|y_j-y_i\|^2 = \sum_{j=1}^k \big[\|y_j\|^2 + \|y_i\|^2 + 2\langle y_i, y_j\rangle\big] = 2k\,R^2 + 2\left\langle y_i, \sum_{j=1}^k y_j\right\rangle = 2k\,R^2.
\]
The sum on the left is a sum of only $k-1$ positive terms since $y_i-y_i=0$, so there exists $j\neq i$ such that $\|y_i - y_j\|^2 \geq 2\,\frac{k}{k-1}\,R^2$. On the other hand, we know that $z_i, z_j$ coincide in all but two coordinates, so by \eqref{eq alpha, beta} we find that
\[
\|z_i - z_j\|^2 = 2(\alpha-\beta)^2 = 2\left[\sqrt{\frac{k-1}k}- \frac1{\sqrt{k(k-1)}}\right] R^2 = 2\,\frac{k-1}{k}\,R^2\left[1 - \frac1{k-1}\right]^2 = 2\,\frac{k}{k-1}R^2.
\]
In particular, since $\|A\|=1$ we find that 
\begin{equation}\label{eq i, j distance}
2\,\frac{k}{k-1}R^2 \leq \|y_i-y_j\|^2 \leq \|A(y_i-y_j)\|^2 = \|z_i-z_j\|^2 = 2\,\frac{k}{k-1}R^2.
\end{equation}
Since strict inequality cannot hold, we find that \ref{eq i, j distance} must hold for all $1\leq i\neq j\leq k$ and thus $\|y_i-y_j\|^2 = \|z_i-z_j\|^2$. This in particular implies that $\langle y_i, y_j\rangle = \langle z_i, z_j\rangle$ for all $i, j=1,\dots, k$. Since $\{y_1, \dots, y_{k-1}\}$ is a basis of $\R^{k-1}$, we conclude that $A$ is an isometric embedding.
\end{proof}

\subsection{Proofs from Section \ref{section caveats}}

We begin by proving that the maximum margin classifier in the problem under discussion is in fact $f(x) = \frac{x}2$.

\begin{proof}[Proof of Lemma \ref{lemma max margin}]
Note that $\bar f(x) = \frac{f(x) - f(-x)}2$ satisfies 
\[
\xi_x\,\bar f(x) = \frac{\xi_x f(x) + \xi_{-x}f(-x)}2 \geq \min\big\{\xi_x f(x),\, \xi_{-x}f(-x)\big\}
\]
for $\P$-almost every $x$. We can therefore assume that the maximum margin classifier is a odd function. The function class under consideration therefore is the convex hull of the family
\[
\H^\circ = \left\{\frac{a\,\sigma(w x+b) - a\sigma(b-wx)}{2\,|a|\,[|w|+|b|]} : a\neq 0, (w,b)\neq 0\right\}.
\] 
Consider the map
\[
F: \conv(\H^\circ) \to \R, \qquad F(h) = h(1)
\]
which bounds the maximum margin functional from above: $\min_{x\in \spt\,\P}\big(\xi_x h(x)\big) \leq 1\cdot h(1)$. 
Since $F$ is linear, it attains its maximum at the boundary of the class, i.e.\ there exist $(w,b)$ such that
\[
\frac{\sigma(w+b) - \sigma(b-w)}{2\,[|w|+|b|]} = F\left(\frac{\sigma(wx+b)- \sigma(b-wx)}{2\,[|w| + |b|]}\right) = \max_{h\in \conv(\H^\circ)} F(h)
\]
and thus
\[
\max_{h\in \conv(\H^\circ)} \min_{x\in \spt\,\P}\big(\xi_x h(x)\big) = \max_{w,b} \frac{\sigma(w+b) - \sigma(b-w)}{2\,[|w|+|b|]}  \leq \frac{\sigma(w+b)}{2\,[|w|+|b|]} \leq \frac12.
\]
The bound is realized precisely if and only if $w> b>0$, i.e.\ due to the positive homogeneity of ReLU if and only if
\[
h(x) = \frac{\sigma(x+b) - \sigma(b-x)}{2\,\big[1+|b|\big]} = \frac1{2\big[1+|b|\big]}\begin{cases}{x+b}&x>b\\ 2x& -b<x<b\\{ x-b}& x< -b\end{cases}
\]
for $b\in[0,1]$.
\end{proof}

Finally, we prove the non-collapse result in the three neuron model.

\begin{proof}[Proof of Lemma \ref{lemma convergence three neurons}]
The gradient flow equation is the ODE
\[
\begin{pmatrix} \dot a_1 \\ \dot a_2 \\ \dot a_3\end{pmatrix} = \begin{pmatrix} p_1 \,\exp(-a_1)\\ p_2\,\exp(-a_2) - p_3\,\exp(a_2-a_3)\\ p_3\,\exp(a_2-a_3)\end{pmatrix}
\]
The first equation is solved easily explicitly since 
\[
\frac{d}{dt} \exp(a_1) = \exp(a_1)\,\dot a_1 = p_1\qquad\Ra\quad a_1(t) = \log\left(e^{a_1(0)} + p_1t\right).
\]
The second equation can be reformulated as
\[
\frac{d}{dt} \exp(a_2) = \exp(a_2)\,\dot a_2 = p_2 - p_3\,\exp(2a_2-a_3),
\]
which leads us to consider
\begin{align*}
\frac{d}{dt} \exp(2 a_2 - a_3) &= \exp(2 a_2 - a_3)\big[2\,\dot a_2 - \dot a_3\big]\\
	&= \exp(2 a_2 - a_3)\big[ 2p_2 \,\exp(-a_2) - 2\,p_3\,\exp(a_2-a_3) - p_3\,\exp(a_2-a_3)\big]\\
	&= \exp(2 a_2 - a_3)\big[ 2p_2 - 3\,p_3\,\exp(2a_2-a_3)\big]\,\exp(-a_2).
\end{align*}
Denote $f(t) = \exp(2 a_2 - a_3)$. The differential equation
\begin{equation}\label{eq dynamics of factor}
f' = f\,(2p_2 - 3\,p_3 f)\,\exp(-a_2)
\end{equation}
implies that $f\equiv \frac{2p_2}{3p_3}$ if $f(0) = \frac{2p_2}{3p_3}$. The same is true for long times and arbitrary initialization (anticipating that the integral of $\exp(-a_2)$ diverges). If the equality is satisfied exactly, we find that
\[
\frac{d}{dt} \exp(a_2) = p_2 - p_3\,\exp(2a_2-a_3) = p_2 - p_3\,\frac{2p_2}{3p_3} = \frac{p_2}3\qquad \Ra \quad a_2(t) = \log\left(e^{a_2(0)} + \frac {p_2}3t\right)
\]
and thus
\[
\exp(2a_2-a_3) = \frac{2p_2}{3p_3} \quad \Ra\quad \exp(a_3) = \frac{3p_3}{2p_2}\,\exp(2 a_2)\quad\Ra\quad a_3 = \log\left(\frac{3p_3}{2p_2}\,\exp(2 a_2)\right) = \log\left(\frac{3p_3}{2p_2}\right) + 2 a_2
\]
The question is whether all data points in the same class are mapped to the same value. This is only a relevant question for the `outer' class where
\begin{align*}
f(t, -1) &= a_1(t)\\
	&= \log\left(e^{a_1(0)} + p_1t\right)\\
f(t,1) &= (a_3-a_2)(t)\\
	&= \log\left(\frac{3p_3}{2p_2}\right) + a_2(t)\\
	&= \log\left(\frac{3p_3}{2p_2}\right) + \log\left(e^{a_2(0)} + \frac {p_2}3t\right)
\end{align*}
In particular
\begin{align*}
f(t,1) - f(t,-1) &= \log\left(\frac{3p_3}{2p_2}\right) + \log\left(e^{a_2(0)} + \frac {p_2}3t\right) - \log\left(e^{a_1(0)} + p_1t\right)\\
	&= \log\left(\frac{3p_3}{2p_2}\right) + \log\left(\frac{ e^{a_2(0)} + \frac {p_2}3t }{ e^{a_1(0)} + p_1t }\right)\\
	&\to \log\left(\frac{3p_3}{2p_2}\right) + \log\left(\frac{p_2}{3\,p_1}\right)\\
	&= \log \left(\frac{3p_3}{2p_2}\frac{p_2}{3\,p_1}\right)\\
	&= \log \left( \frac{p_3}{2p_1}\right)
\end{align*}
Thus the difference between $f(t, 1)$ and $f(t, -1)$ goes to zero if and only if $p_3=2p_1$.

Finally, we remark that if $\exp(2a_2-a_3) = \frac{2p_2}{3\,p_3}$ is not satisfied exactly at time $t=0$, then by \eqref{eq dynamics of factor}, we find that it is approximately satisfied at a later time $t_0\gg1$. Since the influence of the initial condition goes to zero, we find that the conclusion is almost satisfied by considering dynamics starting at $(a_1, a_2, a_3)(t_0)$. This argument can easily be made quantitative.
\end{proof}

\end{document}